\documentclass{article}[11pt,onecolumn,letter]
\usepackage{fullpage}
\usepackage{graphics,graphicx,subfigure,psfrag}
\usepackage{mymath}
\graphicspath{{./}}

\newcommand{\comment}[1]{}

\usepackage{color,umoline,correct}
	\renewcommand{\correct}{\ShowUpdate}

\usepackage{authblk,url}

\begin{document}

\title{Asymptotic Analysis of Generative Semi-Supervised Learning}
\author{Joshua V Dillon\textsuperscript{*}}
\author{Krishnakumar Balasubramanian}
\author{Guy Lebanon}
\affil{
	School of Computational Science \& Engineering \\ 
	College of Computing \\
	Georgia Institute of Technology \\ Atlanta, Georgia
}
\date{\today}
\maketitle
\thispagestyle{empty}
\let\oldthefootnote\thefootnote
\renewcommand{\thefootnote}{\fnsymbol{footnote}}
\footnotetext[1]{To whom correspondence should be addressed. Email: \url{jvdillon@gatech.edu}}
\let\thefootnote\oldthefootnote

\begin{abstract}
Semisupervised learning has emerged as a popular framework for improving
modeling accuracy while controlling labeling cost. Based on an extension of
stochastic composite likelihood we quantify the asymptotic accuracy of
generative semi-supervised learning.
\correct{This allows us to measure}
{In doing so, we complement distribution-free analysis by providing an
alternative framework to measure}
the value associated with different labeling policies and resolve the
fundamental question of how much data to label and in what manner.  We
demonstrate our \correct{framework}{approach} with both simulation studies and real world
experiments using naive Bayes for text classification and MRF\correct{/}{s and }CRFs for
structured prediction in NLP. 
\end{abstract}

\section{Introduction}
Semisupervised learning (SSL) is a technique for estimating statistical models
using both labeled and unlabeled data. It is particularly useful when the costs
of obtaining labeled and unlabeled samples are different. In particular,
assuming that unlabeled data is more easily available\correct{}{,} SSL provides
improved modeling accuracy by adding a large number of unlabeled samples to a
relatively small labeled dataset.

\correct{Despite the fact that SSL is an important research area, there have
been relatively few successes in mathematically quantifying how much SSL
improves modeling beyond traditional supervised techniques.}{ The practical
value of SSL has motivated several attempts to mathematically quantify its
value beyond traditional supervised techniques.  }
Of particular importance is the dependency of that improvement on the amount of
unlabeled and labeled data.  In the case of structured prediction the accuracy
of the SSL estimator depends also on the specific manner in which sequences are
labeled. Focusing on the framework of generative or likelihood-based SSL
applied to classification and structured prediction we identify the following
questions which we address in this paper.

\noindent \emph{Q1: Consistency (classification).} What combinations of labeled
and unlabeled data lead to precise models in the limit of large data.

\noindent \emph{Q2: Accuracy (classification).} How can we quantitatively
express the estimation accuracy for a particular generative model as a function
of the amount of labeled and unlabeled data. What is the improvement in
estimation accuracy resulting from replacing an unlabeled example with a
labeled one. 

\noindent \emph{Q3: Consistency (structured prediction).} What strategies for
sequence labeling lead to precise models in the limit of large data.

\noindent \emph{Q4: Accuracy (structured prediction).} How can we
quantitatively express the estimation quality for a particular model and
structured labeling strategy. What is the improvement in estimation accuracy
resulting from replacing one labeling strategy with another.

\noindent \emph{Q5: Tradeoff (classification and structured prediction).} How
can we quantitatively express the tradeoff between the two competing goals of
improved prediction accuracy and low labeling cost. What are the possible ways
to resolve that tradeoff optimally within a problem-specific context.

\noindent \emph{Q6: Practical Algorithms.} How can we determine how much data
to label in practical settings. 
  
The first five questions are of fundamental importance to SSL theory. Recent
related work has concentrated on large deviation bounds for discriminative SSL
as a response to Q1 and Q2 above. While enjoying broad applicability, such
non-parametric bounds \correct{may not be very informative due to their
generality as they apply to a wide range of models and situations}{are weakened
when the model family's worst-case is atypical}. \correct{Our approach
complements these efforts by providing asymptotic analysis without the use of
bounds that apply to the specific generative models under consideration.}{By
forgoing finite sample analysis, our approach complements these efforts and
provides insights which apply to the specific generative models under
consideration.} \correct{The last question above, Q6, complements the
mathematical analysis of Q1-Q5 with practical algorithms that can be applied in
practice.}{In presenting answers to the last question, we reveal the relative
merits of asymptotic analysis and how its employ, perhaps surprisingly, renders
practical heuristics for controlling labeling cost.}

Our asymptotic \correct{derivation}{derivations} are possible by extending the
recently proposed stochastic composite likelihood formalism~\cite{Dillon2009a}
and showing that generative SSL is a special case of that extension. The
\correct{derivations}{implications of this analysis} are demonstrated using a
simulation study as well as text classification and NLP structured prediction
experiments.  \correct{Much of the}{The} developed framework, however, is
general enough to apply to any generative SSL
problem.  \correct{}{As in \cite{Liang2008}, the delta method
transforms our results from parameter asymptotics to prediction risk asymptotics. We omit these results for lack of space.}

\section{Related Work} 
Semisupervised learning has received much attention in the past decade. Perhaps
the first study in this area was done by Castelli and Cover \cite{Castelli1996} who examined the convergence of the classification error rate as
a labeled example is added to an unlabeled dataset drawn from a Gaussian
mixture model. Nigam et al.\ \cite{Nigam2000} proposed a practical SSL framework
based on maximizing the likelihood of the observed data. An edited volume describing more recent developments is \cite{ChaSchZie06}.

The goal of theoretically quantifying the effect of SSL has \correct{}{recently} gained increased
attention. Sinha and Belkin~\cite{Sinhanips07} examined the effect of
using unlabeled samples with imperfect models for mixture models. Balcan and Blum~\cite{Balcan2010} and Singh et al. \cite{singh2008unlabeled} analyze
discriminative SSL using PAC theory and large deviation bounds.  
\correct{}{
Additional analysis has been conducted under specific distributional
assumptions such as the ``cluster assumption'',
``smoothness assumption'' and the ``low density assumption.''\cite{ChaSchZie06}
However, many of these assumptions are criticized in \cite{Bendavid2008}.}

Our work complements the above studies in that we focus on generative as
opposed to discriminative SSL. In contrast to most other studies, we derive
model specific asymptotics as opposed to non-parametric large deviation bounds.
While such bounds are helpful as they apply to a broad set of cases, they also
provide less information than model-based analysis due to their generality. Our
analysis, on the other hand, requires knowledge of the specific model family
and an estimate of the model parameter. The resulting asymptotics, however,
apply specifically to the case at hand without the need of potentially loose
bounds.

We believe that our work is the first to consider and answer questions Q1-Q6 in
the context of generative SSL. In particular, our work provides a new framework
for examining the accuracy-cost SSL tradeoff in a way that is quantitative,
practical, and model-specific.
  
\section{Stochastic SSL Estimators}
Generative SSL \cite{Nigam2000,ChaSchZie06} estimates a parametric model by maximizing the observed likelihood incorporating $L$ labeled and $U$ unlabeled examples 
\begin{align}\label{eq:ll0}
\ell(\theta) = \sum_{i=1}^L \log p_{\theta}(X^{(i)},Y^{(i)}) + \sum_{i=L+1}^{L+U} \log p_{\theta}(X^{(i)})
\end{align} 
where $p_{\theta}(X^{(i)})$ above is obtained by marginalizing the latent label $\sum_y p_{\theta}(X^{(i)},y)$. A classical example is the naive Bayes model in \cite{Nigam2000} where $p_{\theta}(X,Y)=p_{\theta}(X|Y) p(Y)$,  $p_{\theta}(X|Y=y)=\text{Mult}([\theta_y]_1,\ldots,[\theta_y]_V)$. The framework, however, is general enough to apply to any generative model $p_{\theta}(X,Y)$. 
  
To analyze the asymptotic behavior of the maximizer of \eqref{eq:ll0} we assume that the ratio between labeled to unlabeled examples $\lambda=L/(L+U)$ is kept constant while $n=L+U\to\infty$. More generally, we assume a stochastic version of \eqref{eq:ll0} where each one of the $n$ samples $X^{(1)},\ldots,X^{(n)}$ is labeled with probability $\lambda$
\begin{align} \label{eq:ll1}
\ell_n(\theta)&=\sum_{i=1}^n Z^{(i)} \log p_{\theta}(X^{(i)},Y^{(i)}) 
+ \sum_{i=1}^n (1-Z^{(i)}) \log p_{\theta}(X^{(i)}), \quad Z^{(i)}\sim \text{Bin}(1,\lambda). 
\end{align}
The variable $Z^{(i)}$ above is an indicator taking the value 1 with probability $\lambda$ and 0 otherwise. Due to the law of large numbers for large $n$ we will have approximately $L=n\lambda$ labeled samples and $U=n(1-\lambda)$ unlabeled samples thus achieving the asymptotic behavior of \eqref{eq:ll0}. 

Equation~\eqref{eq:ll1} is sufficient to handle the case of classification. However, in the case of structured prediction we may have sequences $X^{(i)},Y^{(i)}$ where for each $i$ some components of the label sequence $Y^{(i)}$ are missing and some are observed. For example one label sequence may be completely observed, another may be  completely unobserved, and a third may have the first half labeled and the second half not.

More formally, we assume the existence of a sequence labeling policy or strategy $\wp$ which maps label sequences  $Y^{(i)}=(Y^{(i)}_1,\ldots,Y^{(i)}_m)$ to a subset corresponding to the observed labels $\wp(Y^{(i)}) \subset \{Y^{(i)}_1,\ldots,Y^{(i)}_m\}$. To achieve full generality we allow the labeling policy $\wp$ to be stochastic, leading to different subsets of $\{Y^{(i)}_1,\ldots,Y^{(i)}_m\}$ with different probabilities. A simple ``all or nothing'' labeling policy could label the entire sequence with probability $\lambda$ and otherwise ignore it. Another policy may label the entire sequence, the first half, or ignore it completely with equal probabilities
\begin{align}
\wp(Y)\!=\!\begin{cases}
Y^{(i)}_1,\ldots,Y^{(i)}_m &\!\!\!\!\!\! \text{ with probability }1/3\\
\emptyset &\!\!\!\!\!\! \text{ with probability }1/3\\
Y^{(i)}_1,\ldots,Y^{(i)}_{\lfloor m/2\rfloor} &\!\!\!\!\!\! \text{ with probability }1/3
\end{cases}. \label{eq:policyExample}
\end{align}

We thus have the following generalization of \eqref{eq:ll1} for structured prediction 
\begin{align} \label{eq:ll2}
\ell_n(\theta) &= \sum_{i=1}^n \log p_{\theta}(\wp(Y^{(i)}),X^{(i)}).
\end{align}
Equation~\eqref{eq:ll2} generalizes standard SSL from all or nothing labeling to arbitrary labeling policies. The fundamental SSL question in this case is not simply what is the dependency of the estimation accuracy on $n$ and $\lambda$. Rather we ask what is the dependency of the estimation accuracy on the labeling policy $\wp$. Of particular interest is the question what labeling policies $\wp$ achieve high estimation accuracy coupled with low labeling cost. Answering these questions leads to a generative SSL theory that quantitatively balances estimation accuracy and labeling cost.

Finally, we note that both \eqref{eq:ll1} and \eqref{eq:ll2} are random variables whose outcomes depend on the random variables $Z^{(1)},\ldots,Z^{(n)}$ (for \eqref{eq:ll1}) or  $\wp$ (for \eqref{eq:ll2}). Consequentially, the analysis of the maximizer $\hat\theta_n$ of \eqref{eq:ll1} or \eqref{eq:ll2} needs to be done in a probabilistic manner.

\section{A1: Consistency (Classification)} \label{sec:A1}
Assuming that the data is generated from $p_{\theta_0}(X,Y)$ consistency corresponds to the convergence of 
\begin{align} \label{eq:mle}
\hat\theta_n=\argmax_{\theta} \ell_n(\theta)
\end{align}
to $\theta_0$ with probability 1 as $n\to\infty$ ($\ell_n$ is defined in \eqref{eq:ll1}). This implies that in the limit of large data our estimator would converge to the truth. Note that large data $n\to\infty$ in this case means that both labeled and unlabeled data increase to $\infty$ (but their relative sizes remain the constant $\lambda$). 

We show in this section that the maximizer of \eqref{eq:ll1} is consistent assuming that $\lambda>0$. This is not an unexpected conclusion but for the sake of completeness we prove it here rigorously. The proof technique will also be used later when we discuss consistency of SSL estimators for structured prediction. 

The central idea in the proof is to cast the generative SSL estimation problem as an extension of stochastic composite likelihood  \cite{Dillon2009a}. Our proof follows similar lines to the consistency proof of \cite{Dillon2009a} with the exception that it does not assume independence of the indicator functions $Z^{(i)}$ and $(1-Z^{(i)})$ as is assumed there.

\begin{defn}
A distribution $p_{\theta}(X,Y)$ is said to be identifiable if $\theta\neq\eta$ entails that $p_{\theta}(X,Y)-p_{\eta}(X,Y)$ is not identically zero. 
\end{defn}
\begin{prop} \label{prop:consistency1}
Let $\Theta\subset \R^r$ be a compact set, and $p_{\theta}(x,y)>0$ be identifiable and smooth in $\theta$. Then if $\lambda>0$ the maximizer $\hat\theta_n$ of \eqref{eq:ll1} is consistent i.e.,  $\hat\theta_n\to\theta_0$ as $n\to\infty$ with probability 1.
\end{prop}
\begin{proof}
The likelihood function, modified slightly by a linear combination with a constant is $\ell_n'(\theta)=$
\begin{align*}
&\frac{1}{n}\sum_{i=1}^n \left(Z^{(i)} \log p_{\theta}(X^{(i)},Y^{(i)})-\lambda \log p_{\theta_0}(X^{(i)},Y^{(i)})\right) 
+ \frac{1}{n}\sum_{i=1}^n\left((1-Z^{(i)}) \log p_{\theta}(X^{(i)})
-(1-\lambda)\log p_{\theta_0}(X^{(i)})\right),
\end{align*}
converges by the the strong law of large numbers as $n\to\infty$ to its expectation with probability 1
\begin{align*}
\mu(\theta) =-\lambda D(p_{\theta_0}(X,Y)||p_{\theta}(X,Y)) - (1-\lambda) D(p_{\theta_0}(X)||p_{\theta}(X))).
\end{align*}

If we restrict ourselves to the compact set $S=\{\theta: c_1 \leq \|\theta-\theta_0\|\leq c_2\}$ then $|\log p_{\theta}(X,Y)|<K(X,Y)<\infty,\, \forall\theta\in S$. As a result, the conditions for the uniform strong law of large numbers, cf. \correct{}{chapter 16 of} \cite{Ferguson1996}, hold on $S$ leading to 
\begin{align} \label{eq:ulln}
 P\left\{\lim_{n\to\infty} \,\sup_{\theta\in S}\, |\ell_n'(\theta)-\mu(\theta)|=0\right\}=1.
\end{align}

Due to the identifiability of $p_{\theta}(X,Y)$ we have $D(p_{\theta_0}(X,Y) || p_{\theta}(X,Y)) \geq 0$ with equality iff $\theta=\theta_0$. Since also $D(p_{\theta_0}(X)||p_{\theta}(X)))\geq 0$ we have that $\mu(\theta)\leq 0$ with equality iff $\theta=\theta_0$ (assuming $\lambda>0$). Furthermore, since the function $\mu(\theta)$ is continuous it attains its negative supremum on the compact $S$: $\sup_{\theta\in S} \mu(\theta)<0$. 

Combining this fact with \eqref{eq:ulln} we have that there exists $N$ such that for all $n>N$ the likelihood maximizers on $S$ achieves strictly negative values of $\ell_n'(\theta)$ with probability 1. However, since $\ell_n'(\theta)$ can be made to achieve values arbitrarily close to zero under $\theta=\theta_0$, we have that $\hat\theta_n\not\in S$ for $n>N$. Since $c_1,c_2$ were chosen arbitrarily $\hat\theta_n\to \theta_0$ with probability 1.
\end{proof}

The above proposition is not surprising. As $n\to\infty$ the number of labeled examples increase to $\infty$ and thus it remains to ensure that adding an increasing number of unlabeled examples does not hurt the estimator. More interesting is the quantitative description of the accuracy of $\hat\theta_n$ and its dependency on $\theta_0, \lambda, n$ which we turn to next.

\section{A2: Accuracy (Classification)} \label{sec:A2}
The proposition below states that the distribution of the maximizer of \eqref{eq:ll1} is asymptotically normal and provides its variance which may be used to characterize the accuracy of $\hat\theta_n$ as a function of $n,\theta_0,\lambda$. As in Section~\ref{sec:A1} our proof proceeds by casting generative SSL as an extension of stochastic composite likelihood. 

In Proposition~\ref{prop:efficiency1} (below) and in Proposition~\ref{prop:efficiency2} we use $\Var_{\theta_0}(H)$ to denote the variance matrix of a random vector $H$ under $p_{\theta_0}$. The notations $\toop,\tood$ denote convergences in probability and in distribution \cite{Ferguson1996} and $\nabla f(\theta)$, $\nabla^2 f(\theta)$ are the $r\times 1$ gradient vector and $r\times r$ matrix of second order derivatives of $f(\theta)$.
 
\begin{prop} \label{prop:efficiency1}
Under the assumptions of Proposition \ref{prop:consistency1} as well as convexity of $\Theta$ we have the following convergence in distribution of the maximizer of \eqref{eq:ll1} 
\begin{align} \label{eq:efficiency1}
   \sqrt{n}(\hat\theta_n-\theta_0) \tood N\left(0,\Sigma^{-1}\right)
\end{align}
as $n\to\infty$, where
\begin{align*}
\Sigma &= \lambda\Var_{\theta_0}(V_1)+(1-\lambda)\Var_{\theta_0}(V_2)\\
V_1 & = \nabla_{\theta} \log p_{\theta_0}(X,Y), \quad 
V_2  = \nabla_{\theta} \log p_{\theta_0}(X).
\end{align*}
\end{prop}
\begin{proof}
By the mean value theorem and convexity of $\Theta$, there is
$\eta\in(0,1)$ for which $\theta'\!\!=\!\!\theta_0+\eta(\hat\theta_n-\theta_0)$ and 
\begin{align*}
\nabla\ell_n(\hat\theta_n)&=\nabla\ell_n(\theta_0)+ \nabla^2\ell_n(\theta')(\hat\theta_n-\theta_0).
\end{align*}
Since $\hat\theta_n$ maximizes $\ell_n$ we have $\nabla \ell_n(\hat\theta_n)=0$ and 
\begin{align}
\sqrt{n}(\hat\theta_n-\theta_0)&=- \sqrt{n} \left(\nabla^2\ell_n(\theta')\right)^{-1} \left(\nabla\ell_n(\theta_0)\right).
\label{eq:tsell}
\end{align}
By Proposition~\ref{prop:consistency1} we have $\hat\theta_n\toop\theta_0$ which implies that $\theta'\toop\theta_0$ as well.  Furthermore, by the law of large numbers and the fact that $W_n\toop W$ implies $g(W_n)\toop g(W)$ for continuous $g$,
\begin{align}
   (\nabla^2 \ell_n(\theta'))^{-1} &\toop (\nabla^2 \ell_n(\theta_0))^{-1}  \label{eq:res1}\\
   &\toop \Big( \lambda  \E_{\theta_0}  \nabla^2 \log p_{\theta_0}(X,Y) 
+  (1-\lambda)  \E_{\theta_0}  \nabla^2 \log p_{\theta_0}(X) \Big)^{-1} = \Sigma^{-1} \nonumber
\end{align}
where in the last equality we used a well known identity concerning the Fisher information. 

For the remaining term in the rhs of \eqref{eq:tsell} we have 
\begin{align} \label{eq:WQ}
-\sqrt{n} \nabla \ell_n(\theta_0) = -\sqrt{n}\frac{1}{n} \sum_{i=1}^n (W^{(i)}+ Q^{(i)})
\end{align}
where $W^{(i)}=Z^{(i)} \nabla \log p_{\theta_0}(X^{(i)},Y^{(i)})$, $Q^{(i)}=(1-Z^{(i)})\nabla \log p_{\theta_0}(X^{(i)})$. Since \eqref{eq:WQ} is an average of iid random vectors $W^{(i)}+Q^{(i)}$ it is asymptotically normal by the central limit theorem with mean
\begin{align*}
\E_{\theta_0} (Q + W) &= \lambda \E_{\theta_0} \nabla \log p_{\theta_0}(X,Y) + (1-\lambda) \E\nabla \log p_{\theta_0}(X)
= \lambda 0 + (1-\lambda) 0.
\end{align*}
and variance 
\begin{align*}
\Var_{\theta_0}(W+Q) &= \E_{\theta_0} W^2+\E_{\theta_0} Q^2 +2\E_{\theta_0} WQ\\
&= \lambda \Var_{\theta_0} V_1 + (1-\lambda) \Var_{\theta_0} V_2
\end{align*}
where we used \correct{the fact that}{}  $\E(Z(1-Z))=\E Z-\E Z^2=0$ \correct{for binary random variables $Z$}{}.

We have thus established that 
\begin{align} \label{eq:res2}
-\sqrt{n} \nabla \ell_n(\theta_0) \tood N (0,\Sigma).
\end{align}
We finish the proof by combining \eqref{eq:tsell}, \eqref{eq:res1} and \eqref{eq:res2} using Slutsky's theorem. 
\end{proof}

Proposition~\ref{prop:efficiency1} characterizes the asymptotic estimation accuracy using the matrix $\Sigma$. Two convenient one dimensional summaries of the accuracy are the trace and the determinant of $\Sigma$. In some simple cases (such as binary event naive Bayes) $\trace(\Sigma)$ can be brought to a mathematically simple form which exposes its dependency on $\theta_0,n,\lambda$. In other cases the dependency may be obtained using numerical computing.

Figure~\ref{fig:lorem1} displays three error measures for the multinomial naive Bayes SSL classifier \cite{Nigam2000} and the  Reuters RCV1 text classification data. In all three figures the error measures are represented as functions of $n$ (horizontal axis) and $\lambda$ (vertical axis). The error measures are classification error rate (left), trace of the empirical mse (middle), and log-trace of the asymptotic variance (right). The measures were obtained over held-out sets and averaged using cross validation.  Figure~\ref{fig:lorem3} (middle) displays the asymptotic variance as a function of $n$ and $\lambda$ for a randomly drawn $\theta_0$. 

As expected the measures decrease with $n$ and $\lambda$ in all the figures. It
is interesting to note, however, that the shapes of the contour plots are very
similar across the three different measures (top row). This confirms that the
asymptotic variance (right) is a valid proxy for the finite sample measures of
error rates and empirical mse. We thus conclude that the asymptotic variance is an attractive measure that is similar to finite sample error rate and at the same time has a convenient mathematical expression.

\section{A3: Consistency (Structured)} \label{sec:A3}
In the case of structured prediction the log-likelihood \eqref{eq:ll2} is specified using a stochastic labeling policy. In this section we consider the conditions on that policy that ensures estimation consistency, or in other word convergence of the maximizer of \eqref{eq:ll2} to $\theta_0$ as $n\to\infty$.

We assume that the labeling policy $\wp$ is a probabilistic mixture of deterministic sequence labeling functions $\chi_1,\ldots,\chi_k$. In other words, $\wp(Y)$ takes values $\chi_i(Y), i=1,\ldots,k$ with probabilities $\lambda_1,\ldots,\lambda_k$. For example the policy \eqref{eq:policyExample} corresponds to $\chi_1(Y)=Y$, $\chi_2(Y)=\emptyset$, $\chi_3(Y)=\{Y_1,\ldots,Y_{\lfloor m/2\rfloor}\}$ (where $Y=\{Y_1,\ldots,Y_m\}$) and $\lambda=(1/3,1/3,1/3)$. 

Using the above notation we can write \eqref{eq:ll2} as
\begin{align} \label{eq:ll2m}
\ell_n(\theta) &= \sum_{i=1}^n\sum_{j=1}^k Z_j^{(i)} \log p_{\theta}(\chi_j(Y^{(i)}),X^{(i)}) \\
Z^{(i)} &\sim \text{Mult}(1,(\lambda_1,\ldots,\lambda_k)) \nonumber
\end{align}
which exposes its similarity to the stochastic composite likelihood function in \cite{Dillon2009a}. Note however that \eqref{eq:ll2m} is not formally a stochastic composite likelihood since $Z^{(i)}_j, j=1,\ldots,k$ are not independent and since $\chi_j(Y)$ depends on the length of the sequence $Y$ (see for example $\chi_1$ and $\chi_3$ above). We also use the notation $S_j^m$ for the subset of labels provided by $\chi_j$ on length-$m$ sequences
\[ \chi_j(Y_1,\ldots,Y_m)=\{Y_i:i\in S_j^m\}.\]

\begin{defn} \label{def:identifiable}
A labeling policy is said to be identifiable if the following map is injective
\begin{align*}
\bigcup_{m:q(m)>0}\,\,\bigcup_{j=1}^k \{p_{\theta}(\{Y_r: r\in S_j^m\},X)\} \to p_{\theta}(X,Y)
\end{align*}
where $q$ is the distribution of sequences lengths. In other words, there is at most one collection of probabilities corresponding to the lhs above that does not contradict the joint distribution.
\end{defn}
The importance of Definition~\ref{def:identifiable} is that it ensures the recovery of $\theta_0$ from the sequences partially labeled using the labeling policy. For example, a labeling policy characterized by $\chi_1(Y)=Y_1$, $\lambda_1=1$ (always label only the first sequence element) is non-identifiable for most interesting $p_{\theta}$ as the first sequence component is unlikely to provide sufficient information to characterize the parameters associated with transitions $Y_t\to Y_{t+1}$. 

\begin{prop} \label{prop:consistency2}
Assuming the same conditions as Proposition~\ref{prop:consistency1}, and $\lambda_1,\ldots,\lambda_k>0$ with identifiable $\chi_1,\ldots,\chi_k$, the maximizer of \eqref{eq:ll2m} is consistent i.e., $\hat\theta_n\to\theta_0$ as $n\to\infty$ with probability 1.
\end{prop}
\begin{proof}
The log-likelihood \eqref{eq:ll2}, modified slightly by a linear combination with a constant is 
\begin{align*}
\ell_n'(\theta) = \frac{1}{n}\sum_{i=1}^n \sum_{j=1}^k \Big(Z^{(i)}_j \log p_{\theta}(  \chi_j(Y^{(i)}),X^{(i)})   -\lambda_j \log p_{\theta_0}(\chi_j(Y^{(i)}),X^{(i)})\Big).
\end{align*}
By the strong law of large numbers $\ell_n'(\theta)$ converges to its expectation
\begin{align*}
\mu(\theta) &=  -\sum_{j=1}^k\lambda_j \sum_{m>0} q(m) \cdot D(p_{\theta_0}(\{Y_i:i\in S_j^m\},X)||p_{\theta}(\{Y_i:i\in S_j^m\},X)).
\end{align*}

Since $\mu$ is a linear combination of KL divergences with positive weights it is non-negative and is 0 if $\theta=\theta_0$. The identifiability of the labeling policy ensures that $\mu(\theta)>0$ if $\theta\neq\theta_0$. We have thus established that $\ell_n(\theta)$ converges to a non-negative continuous function $\mu(\theta)$ whose maximum is achieved at $\theta_0$. The rest of the proof proceeds along similar lines as Proposition~\ref{prop:consistency2}. 
\end{proof}

Ultimately, the precise conditions for consistency will depend on the parametric family $p_{\theta}$ under consideration. For many structured prediction models such as Markov random fields the consistency conditions are mild. Depending on the precise feature functions, consistency is generally satisfied for every policy that labels contiguous subsequences with positive probability. However, some care need to be applied for models like HMM containing parameters associated with the start label or end label and with models asserting higher order Markov assumptions.

\section{A4: Accuracy (Structured)}\label{sec:A4}
We consider in this section the dependency of the estimation accuracy in structured prediction SSL \eqref{eq:ll2} on $n,\theta_0$ but perhaps most interestingly on the labeling policy $\wp$. Doing so provides insight into not only how much data to label but also in what way.  

\begin{prop} \label{prop:efficiency2}
Under the assumptions of Proposition \ref{prop:consistency2} as well as convexity of $\Theta$ we have the following convergence in distribution of the maximizer of \eqref{eq:ll2m} 
\begin{align} \label{eq:efficiency2}
   \sqrt{n}(\hat\theta_n-\theta_0) \tood N\left(0,\Sigma^{-1}\right)
\end{align}
as $n\to\infty$, where
\begin{align*}
\Sigma^{-1} &= \E_{q(m)} \left\{\sum_{j=1}^k \lambda_j \Var_{\theta_0} (\nabla V_{jm})\right\}\\
V_{jm} &=  \log p_{\theta_0} (\{Y_i:i\in S_j^m\},X).
\end{align*} 
\end{prop}
\begin{proof}
By the mean value theorem and convexity of $\Theta$ there is $\eta\in(0,1)$ for which $\theta'\!=\!\theta_0 \!+\! \eta (\hat\theta_n-\theta_0)$ and 
\[ \nabla \ell_n(\hat\theta_n) = \nabla \ell_n(\theta_0) + \nabla^2 \ell_n(\theta') (\hat\theta_n-\theta_0).\]
Since $\hat\theta_n$ maximizes $\ell$, $\nabla \ell_n(\hat\theta_n)=0$ and
\begin{align} \label{eq:step1}
   \sqrt{n}(\hat\theta_n-\theta_0) = -\sqrt{n} (\nabla^2 \ell_n(\theta'))^{-1} \nabla \ell_n(\theta_0).
\end{align}
By Proposition~\ref{prop:consistency2} we have
$\hat\theta_n\toop\theta_0$ which implies that
$\theta'\toop\theta_0$ as well. Furthermore, by the law of large numbers and
the fact that if $W_n\toop W$ then $g(W_n)\toop g(W)$ for continuous $g$,
\begin{align}\label{eq:res1}
   (\nabla^2 \ell_n(\theta'))^{-1} &\toop (\nabla^2 \ell_n(\theta_0))^{-1}  \\
   &\toop \left(\sum_{m>0} q(m) \sum_{j=1}^k \lambda_j  \E_{\theta_0}  (\nabla^2 V_{jm})\right)^{-1} \nonumber\\
   &= -\left(\sum_{m>0}q(m)\sum_{j=1}^k \lambda_j \Var_{\theta_0}(\nabla V_{jm})\right)^{-1}. \nonumber
\end{align}
where in the last equality we used a well known identity concerning the Fisher information. 

For the remaining term on the rhs of \eqref{eq:step1} we have
\begin{align} \label{eq:remainingTerm}
\sqrt{n}\, \nabla \ell_n(\theta_0) &= \sqrt{n} \frac{1}{n} \sum_{i=1}^n  W_i
\end{align}
where the random vectors 
\begin{align*}
W_i =  \sum_{m>0} 1_{\{\text{length}(Y^{(i)})=m\}} \sum_{j=1}^k   Z_j^{(i)} \nabla V_{jm}^{(i)}
\end{align*}
have expectation 0 due to the fact that the expectation of the score is 0. The variance of $W_i$ is 
\begin{align*}
\Var_{\theta_0} W_i &=\! \E_{\theta_0}\! \sum_{m>0}  1_{\{\text{length}(Y^{(i)})=m\}} \sum_{j=1}^k   Z_j^{(i)}  \nabla V_{jm}^{(i)} \nabla V_{jm}^{(i)\top}\\
&= \sum_{m>0}  q(m)  \sum_{j=1}^k \lambda_j \E \left(\nabla V_{jm}^{(i)}
\nabla V_{jm}^{(i)\top}\right)
\end{align*} 
where in the first equality we used the fact that $Y^{(i)}$ can have only one length and only one of $\chi_1,\ldots,\chi_k$ is chosen. Using the central limit theorem we thus conclude that 
\begin{align*}
   \sqrt{n}\,\nabla \ell_n(\theta_0) \tood
   N\left(0,\Sigma^{-1}\right)
\end{align*}
and finish the proof by combining \eqref{eq:step1}, \eqref{eq:res1}, and \eqref{eq:res2} using Slutsky's theorem.
\end{proof}

Figure~\ref{fig:lorem2} (left, middle) displays the test-set per-sequence
perplexity for the CoNLL2000 chunking task as a function of the total number of
labeled tokens. We used the Boltzmann chain MRF model that is the MRF
corresponding to HMM (though not identical e.g.,  \cite{MacKay1996}). We
consider labeling policies $\wp$ that label the entire sequence with
probability $\lambda$ and otherwise label contiguous sequences of length 5
(left) or leave the sequence fully unlabeled (middle). Lighter nodes indicate
larger $n$ and unsurprisingly show a decrease in the test-set perplexity
\correct{}{as $n$ is increased}.
\correct{Interestingly, the middle figure shows that labeling policies using a
smaller amount of labels may outperform others indicating that a naive choices
of the labeling scheme may be inefficient.}{Interestingly, the middle figure shows that labeling policies using a smaller
amount of labels may outperform other policies.  This further motivates our analysis and indicates that naive choices of $\wp$ may be inefficient, viz.\ inflating labeling cost with negligible accuracy improvement to accuracy (cf.\ also
Sec.~\ref{sec:A5} for how to avoid this pitfall).}

\subsection{Conditional Structured Prediction}
Thus far our discussion on structured prediction has been restricted to generative models such as HMM or Boltzmann chain MRF. Similar techniques, however, can be used to analyze SSL for conditional models such as CRFs that are estimated by maximizing the conditional likelihood. The key to extending the results in this paper to CRFs is to express conditional SSL estimation in a form similar to \eqref{eq:ll2}
\begin{align*}
\hat\theta_n &= \argmax \sum_{i=1}^n \log p_{\theta}(\wp(Y^{(i)})|X^{(i)})
\end{align*}
and to proceed with an asymptotic analysis that extends the classical conditional MLE asymptotics. We omit further discussion due to lack of space but include some experimental results for CRFs. 

Figure~\ref{fig:lorem3} (left) depicts a similar experiment to the one described in the previous section for conditional estimation in CRF models. The figure displays per-sequence perplexity as a function $n$ ($x$ axis) and $\lambda_1$ ($y$ axis). We observe a trend nearly identical to that of the Boltzmann chain MRF (Figure~\ref{fig:lorem2}, left, middle).

\begin{figure*}
\centering
\begin{tabular}{ccc}
\hspace{-.25in}
	\includegraphics[width=.35\textwidth]{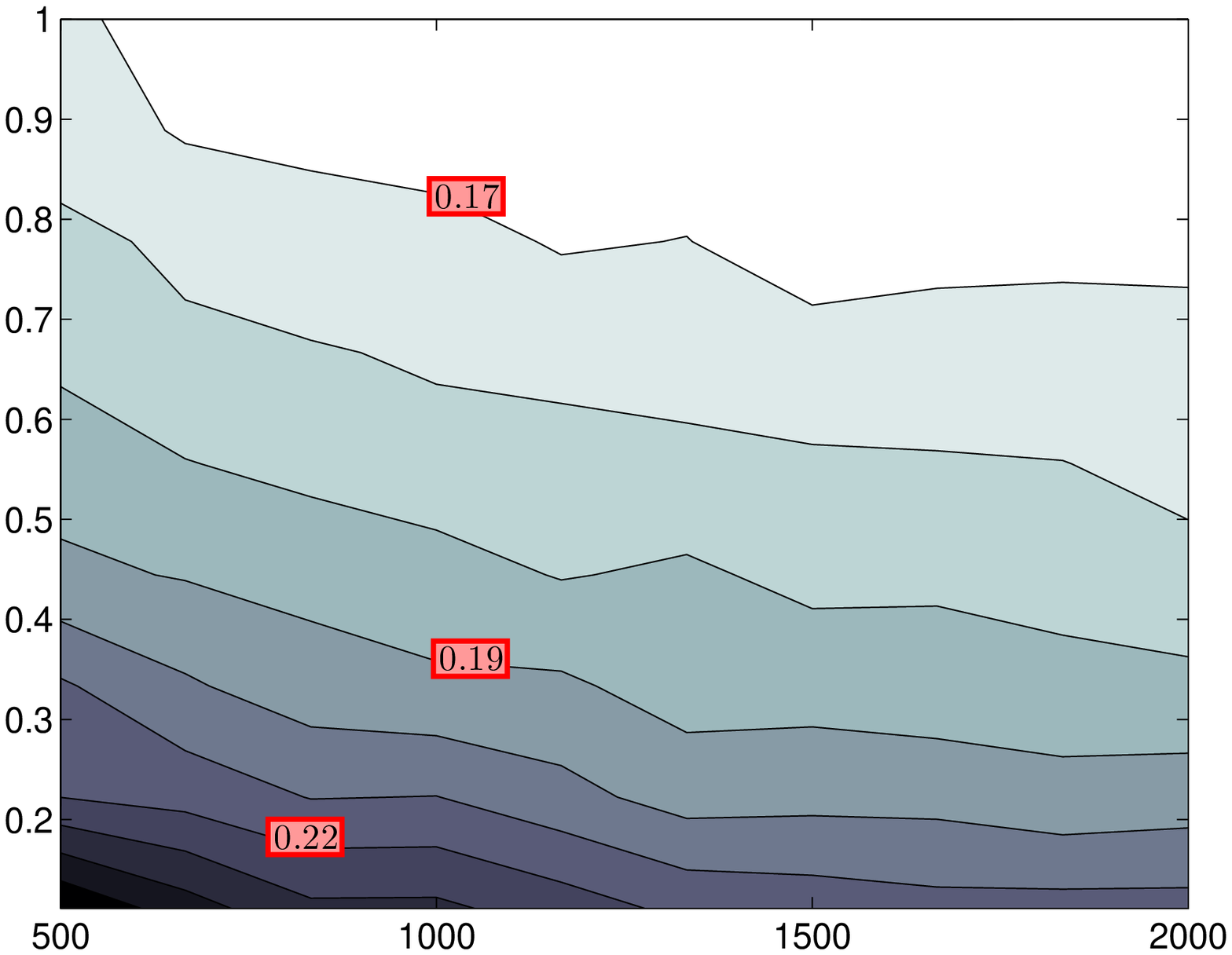} & \hspace{-.20in}
	\includegraphics[width=.344\textwidth]{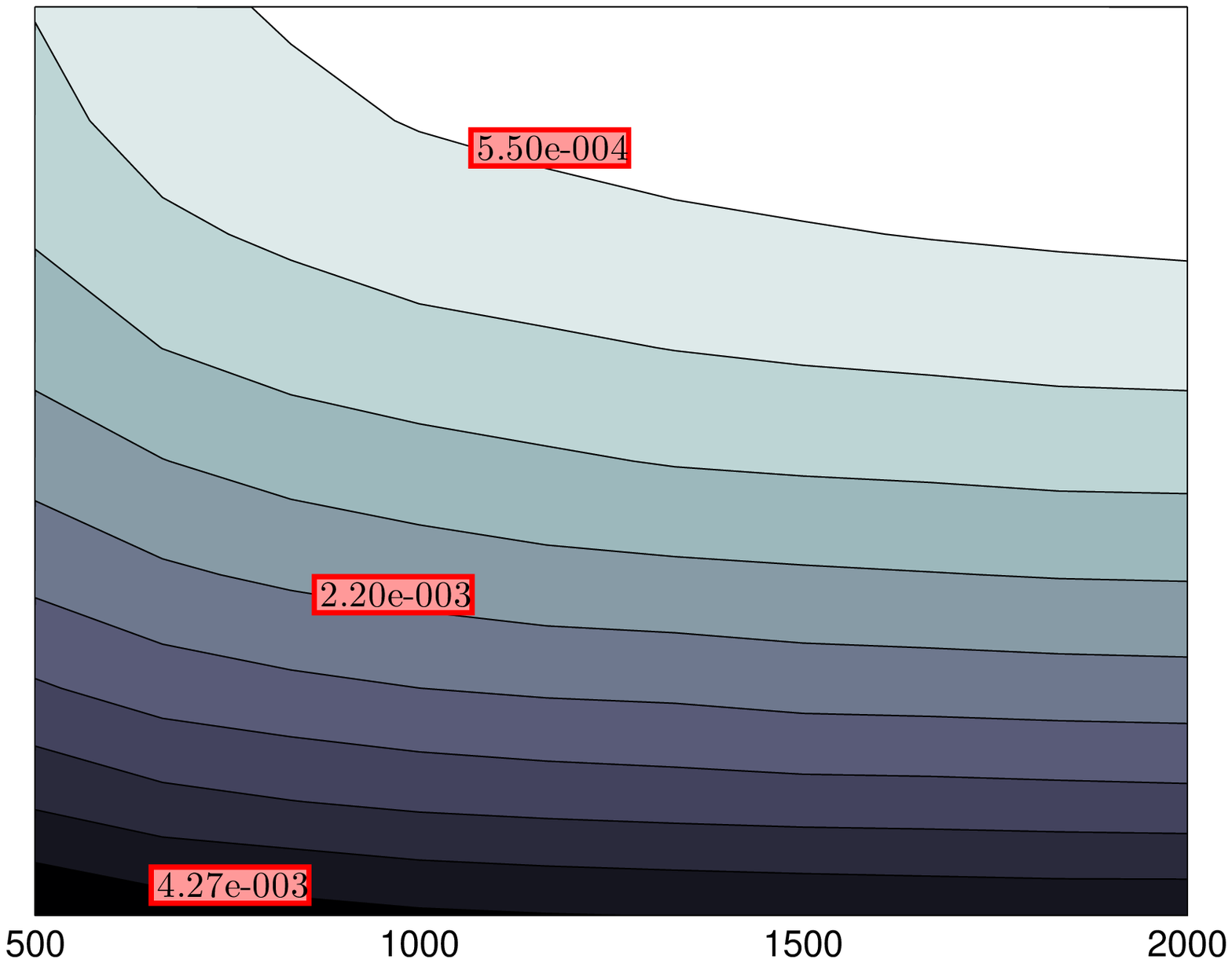}     & \hspace{-.2in}
	\includegraphics[width=.324\textwidth]{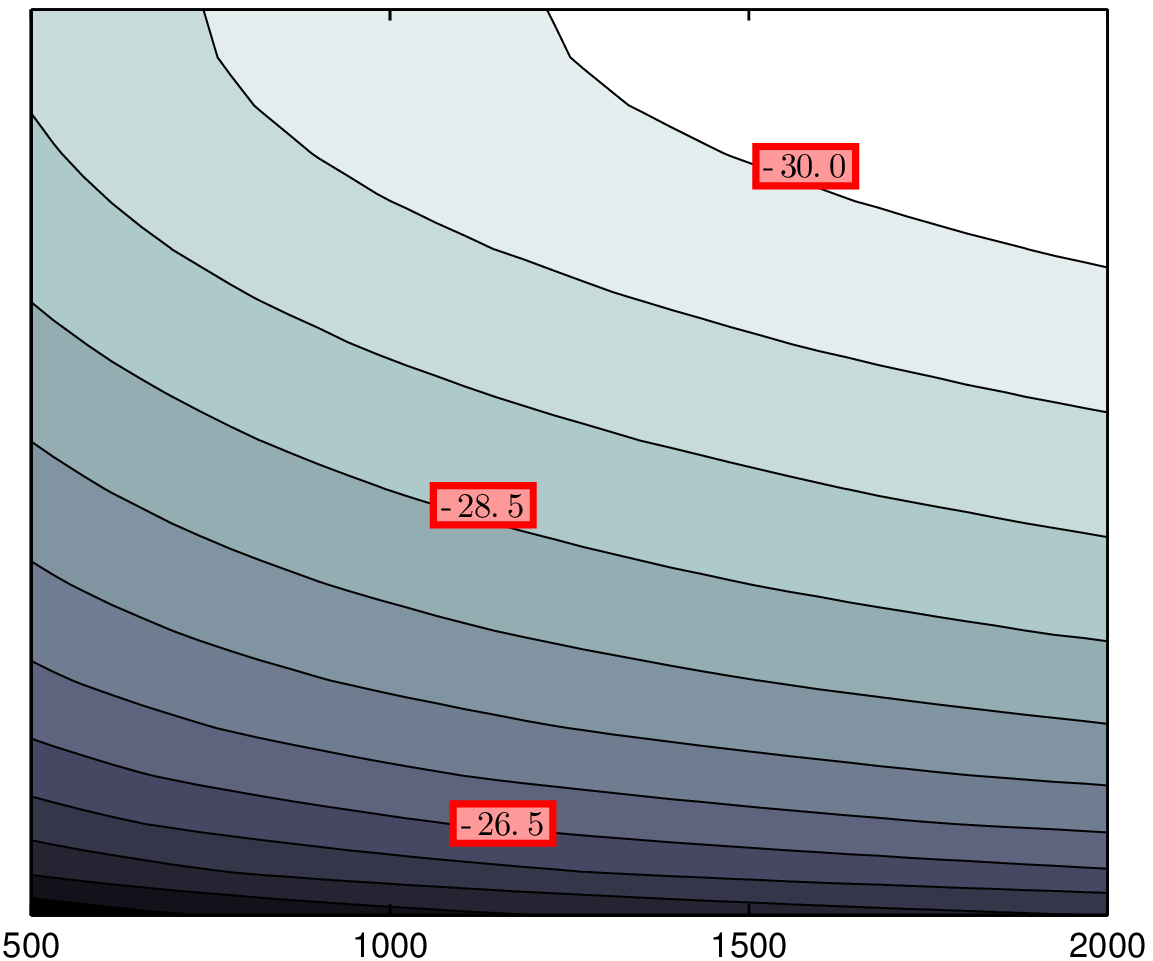} 
\end{tabular}
\caption{	Three error measures for the multinomial naive Bayes SSL classifier applied
	to Reuters RCV1 text data.  In each, error is a function of $n$ (horizontal
	axis) and $\lambda$ (vertical axis).  The left depicts classification error rate, the middle depicts the trace of empirical mse, and right depicts the log-trace of the asymptotic variance.  Results were obtained using held-out sets and averaged using cross validation.  Particularly noteworthy is a striking	correlation among all three figures, justifying the use of asymptotic variance as a surrogate for 	classification error, even for relatively small values of $n$. }
\label{fig:lorem1}
\end{figure*}
\begin{figure*}
\centering
\begin{tabular}[b]{ccc}
\hspace{-.25in}
	\includegraphics[width=.35\textwidth]{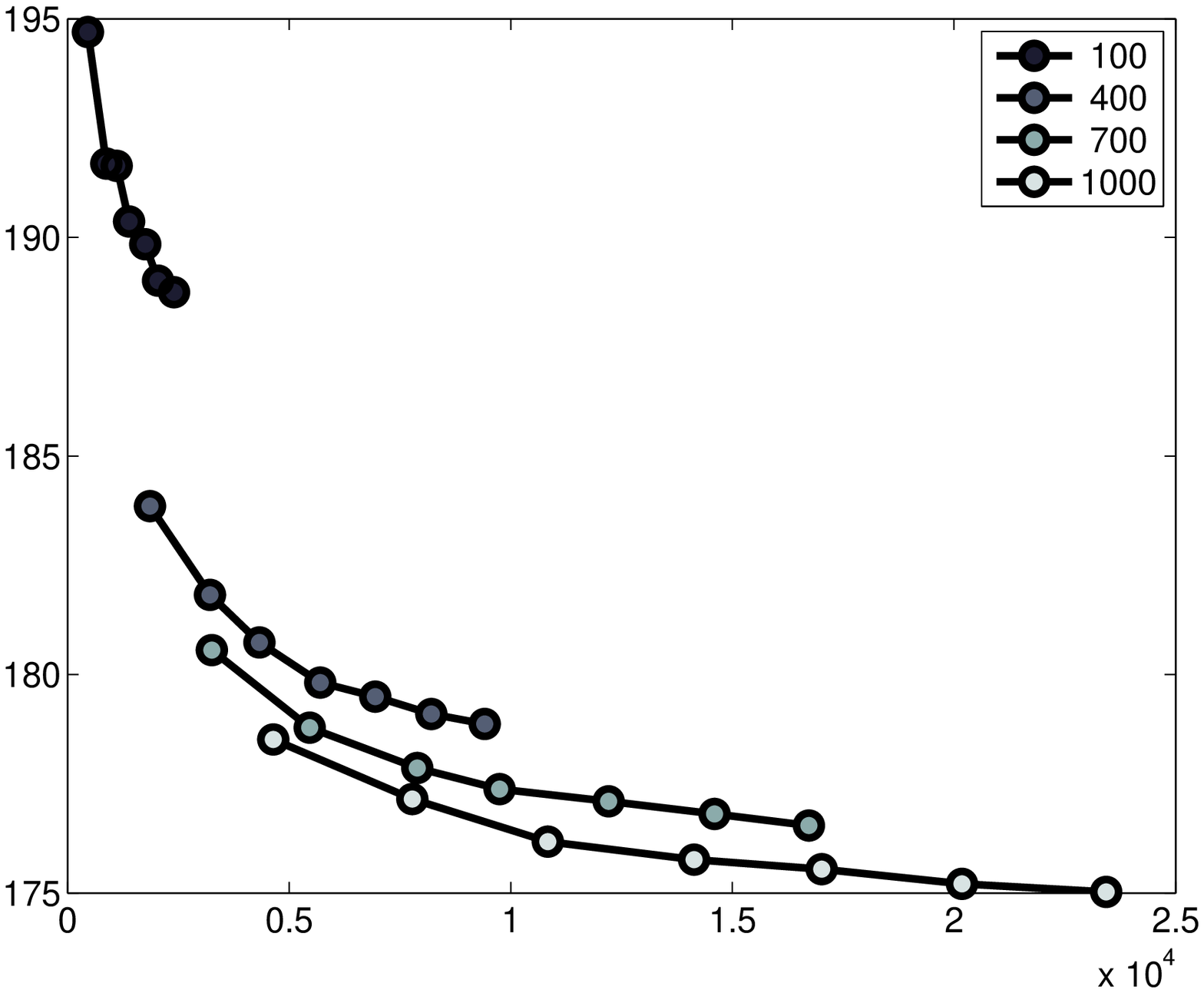}    & \hspace{-.20in}
	\includegraphics[width=.335\textwidth]{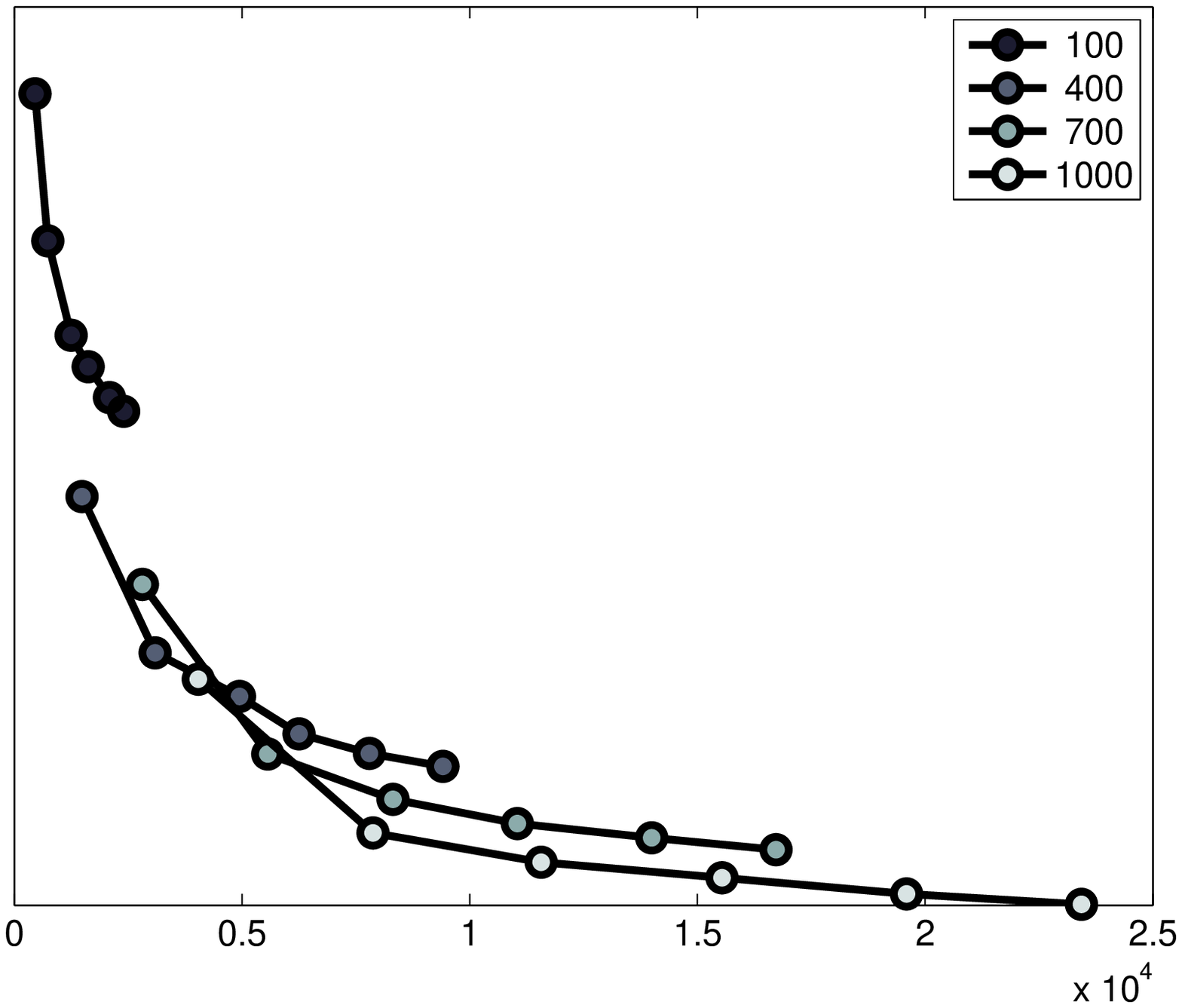}   & \hspace{-.2in}
	\raisebox{1.9ex}{\includegraphics[width=.345\textwidth]{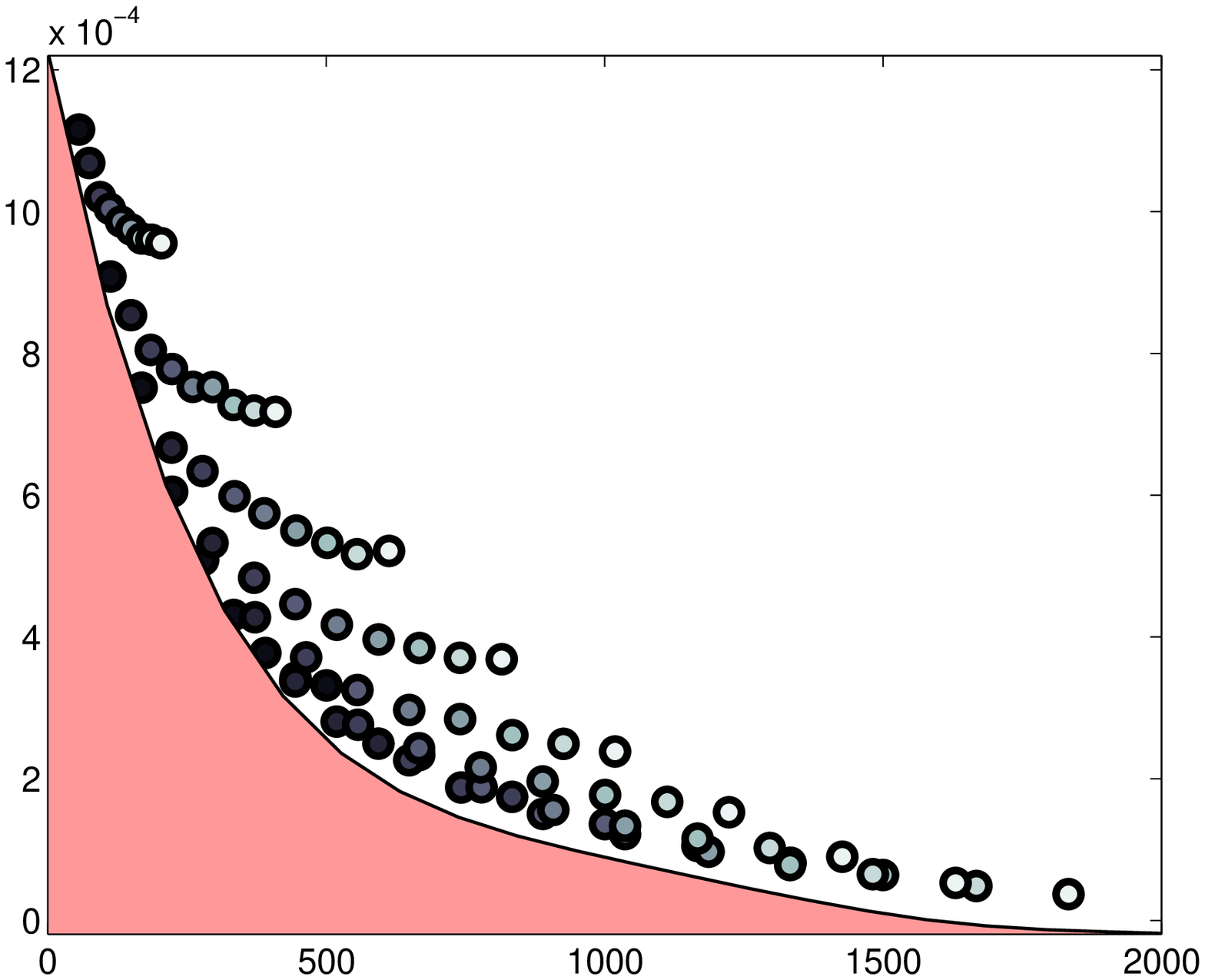}}
\end{tabular}
\caption{Test-set results for two policies of unlabeled data for Boltzmann chain MRFs
applied to the  CoNLL 2000 text-chunking dataset (left, middle).  The \correct{}{shaded portion of the} right panel
depicts the empirically unachievable region for naive Bayes SSL classifier on
the 20-newsgroups dataset.  The left two share a common log-perplexity scale
(vertical axis) while the vertical axis of the right panel corresponds to trace of
the empirical MSE; the horizontal axis indicates labeling cost. As above,
results were obtained using held-out sets and averaged using cross validation.
Collectively these figures represent the application and effect of various
labeling policies.  The left figure depicts the consequence of partially
missing samples for various $n$,$\lambda$ while the middle and right  represent
SSL in the more traditional all or nothing sense: either labeled or unlabeled samples. See text for more details.}
\label{fig:lorem2}
\end{figure*}
\begin{figure*}
\centering
\begin{tabular}{ccc}
\hspace{-.25in}
	\raisebox{0.2ex}{\includegraphics[width=.349\textwidth]{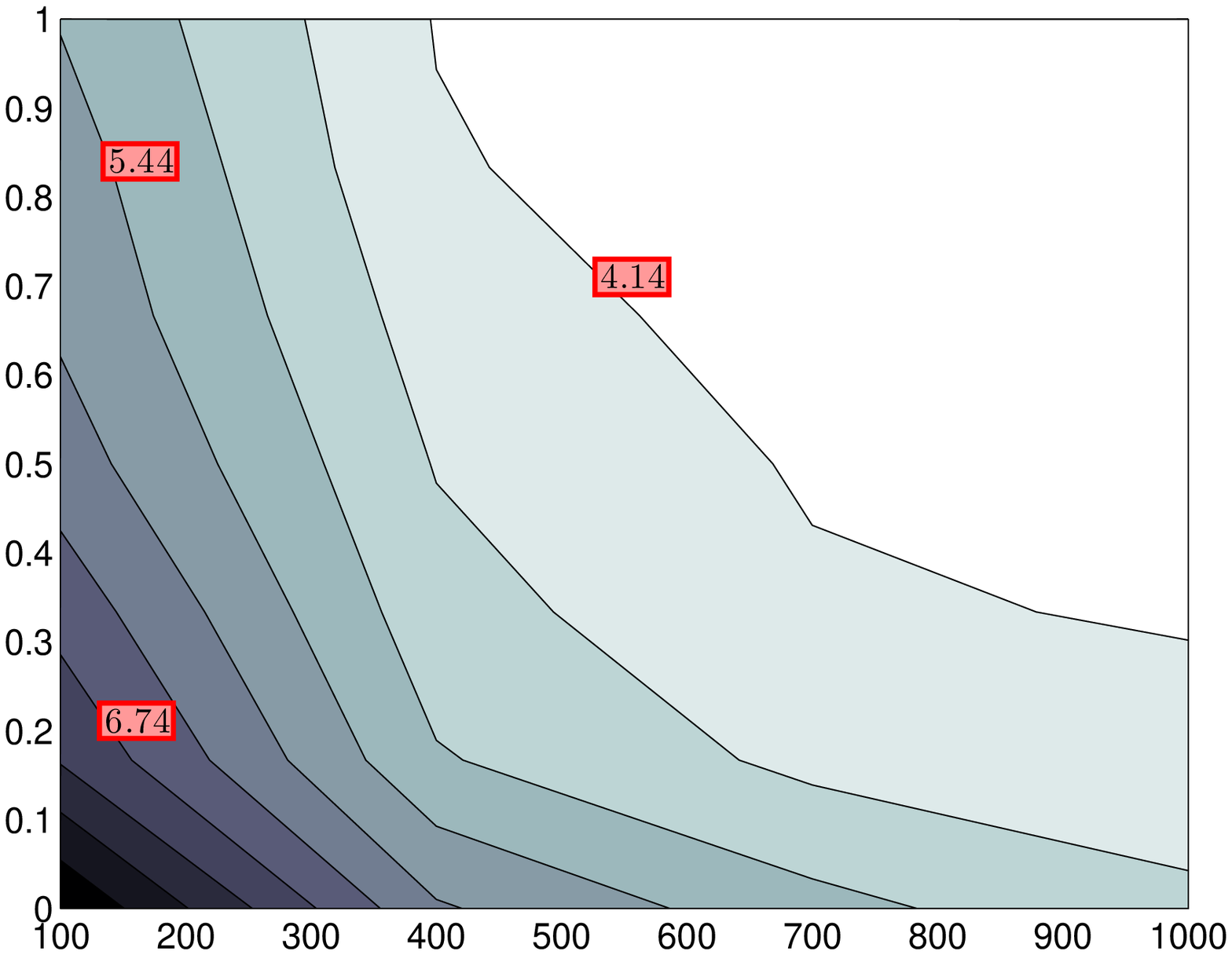}} & \hspace{-.20in}
	\includegraphics[width=.333\textwidth]{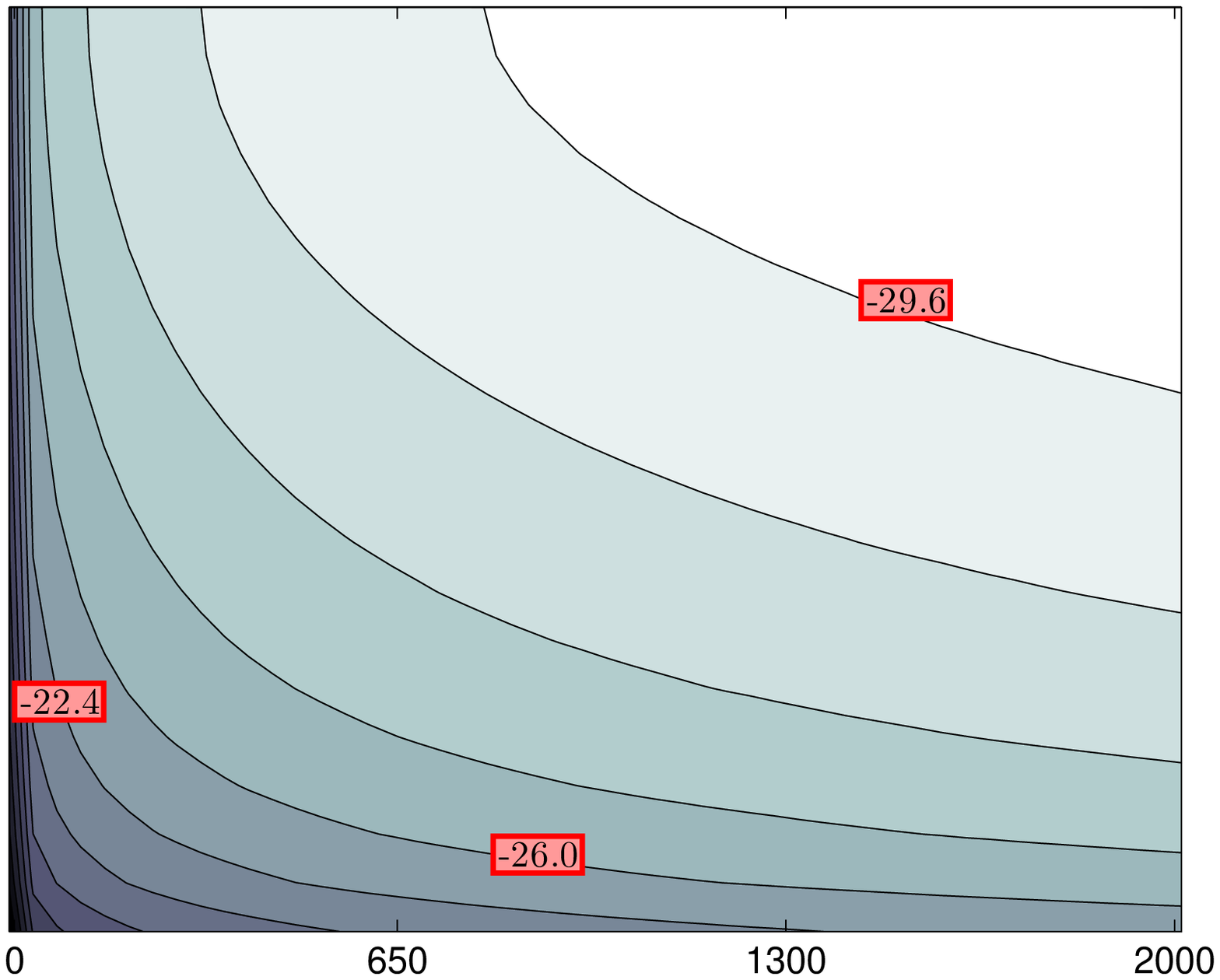}  & \hspace{-.20in}
	\includegraphics[width=.335\textwidth]{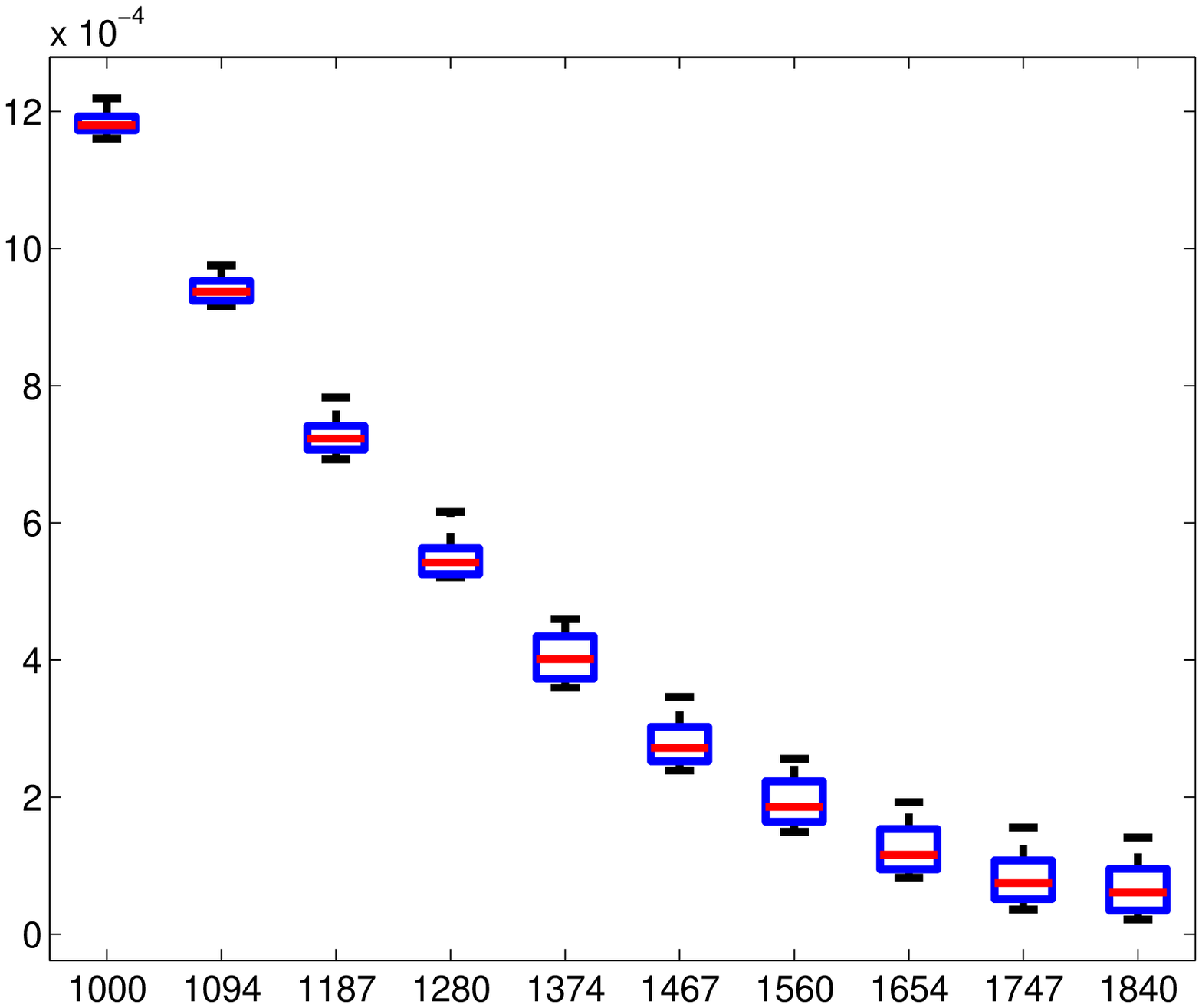}  
\end{tabular}
\caption{Left figure depicts sentence\correct{}{-wise} log-perplexity for CRFs
under the same policy and experimental design of the \correct{(directly)}{} above Boltzmann
chain.
Center figure \correct{exhibits the theoretical asymptotic variance for a
mixture of multinomials.} {represents log-trace of the theoretical variance
and demonstrates phenomena under a simplified scenario, i.e., a mixture of two
$1000$-dim multinomials with unbalanced prior.}
Rightmost figure demonstrates the practical applicability of utilizing
asymptotic analysis to \correct{quantify the extent additional labeling
facilitates improvement in accuracy}{characterize parameter error as a function
of size of training-set partition}.  \correct{}{The training-set is fixed at
$2000$ samples and split for training and validating. As the proportion used
for training is increased, we see a decrease in error. See text for more details.
}}
\label{fig:lorem3}
\end{figure*}

\section{A5: Tradeoff} \label{sec:A5}
As the figures in the previous sections display, the estimation accuracy increases with the total number of labels. The Cramer-Rao lower bound states that the highest accuracy is obtained by the maximum likelihood operating on fully observed data. However, assuming that a certain cost is associated with labeling data SSL resolves a fundamental accuracy-cost tradeoff. A decrease in estimation accuracy is acceptable in return for decreased labeling cost. 

Our ability to mathematically characterize the dependency of the estimation accuracy on the labeling cost leads to a new quantitative formulation of this tradeoff. Each labeling policy ($\lambda,n$ in classification and $\wp$ in structured prediction) is associated with a particular estimation accuracy via Propositions~\ref{prop:efficiency1} and \ref{prop:efficiency2} and with a particular labeling cost. The precise way to measure labeling cost depends on the situation at hand, but we assume in this paper that the labeling cost is proportional to the numbers of labeled samples (classification) and of labeled sequence elements  (structured prediction). This assumption may be easily relaxed by using other labeling cost functions e.g, obtaining unlabeled data may incur some cost as well.

Geometrically, each labeling policy may thus be represented in a two dimensional scatter plot where the horizontal and vertical coordinates correspond to labeling cost and estimation error respectively. Three such scatter plots appear in Figure~\ref{fig:lorem2} (see Section~\ref{sec:A4} for a description of the left and middle panels). 
The right panel corresponds to multinomial naive Bayes SSL classifier and the 20-newsgroups classification dataset. Each point in that panel  corresponds to different $n,\lambda$.  

The origin corresponds to the most desirable (albeit unachievable) position in the scatter plot representing zero error at no labeling cost. The cloud of points obtained by varying $n,\lambda$ (classification) and $\wp$ (structured prediction) represents the achievable region of the diagram. Most attractive is the lower and left boundary of that region which represents labeling policies that dominate others in both accuracy and labeling cost.  The non-achievable region is below and to the left of that boundary (see shaded region in Figure~\ref{fig:lorem2}, right). The precise position of the optimal policy on the boundary of the achievable region depends on the relative importance of minimizing estimation error and minimizing labeling cost. A policy that is optimal in one context may not be optimal in a different context.   

It is interesting to note that even in the case of naive Bayes classification (Figure~\ref{fig:lorem2}, right) some labeling policies (corresponding to specific choices of $n,\lambda$) are suboptimal. These policies correspond to points in the interior of the achievable region. A similar conclusion holds for Boltzmann chain MRF. For example, some of the points in Figure~\ref{fig:lorem2} (left) denoted by the label 700 are dominated by the more lightly shaded points.

We consider in particular three different ways to define an optimal labeling policy (i.e., determining how much data to label) on the boundary of the achievable region
\begin{align} \label{eq:tradeoff1}
(\lambda^*,n^*)_1 &= \argmin_{(\lambda,n):\lambda n \leq C} \trace(\Sigma^{-1})\\
(\lambda^*,n^*)_2 &= \argmin_{(\lambda,n):\trace(\Sigma^{-1}) \leq C} \lambda n \label{eq:tradeoff2}\\
(\lambda^*,n^*)_3 &= \argmin_{(\lambda,n)} \lambda n  + \alpha\, \trace(\Sigma^{-1}). \label{eq:tradeoff3}
\end{align}
The first applies in situations where the labeling cost is bounded by a certain available budget. The second applies when a certain estimation accuracy is acceptable and the goal is to minimize the labeling cost. The third considers a more symmetric treatment of the estimation accuracy and labeling cost.

Equations~\eqref{eq:tradeoff1}-\eqref{eq:tradeoff3} may be easily generalized to arbitrary labeling costs $f(n,\lambda)$.  Equations~\eqref{eq:tradeoff1}-\eqref{eq:tradeoff3} may also be generalized to the case of structured prediction with $\wp$ replacing  $(\lambda,n)$ and $\text{cost}(\wp)$ replacing $\lambda n$.

\section{A6: Practical Algorithms}
Choosing a policy $(\lambda,n)$ or $\wp$ resolves the SSL tradeoff of accuracy vs. cost. Such a resolution is tantamount to answering the basic question of how many labels should be obtained (and in the case of structured prediction also which ones). Resolving the tradeoff via \eqref{eq:tradeoff1}-\eqref{eq:tradeoff3} or in any other way, or even simply evaluating the asymptotic accuracy $\trace(\Sigma)$ requires knowledge of the model parameter $\theta_0$ that is generally unknown in practical settings.
  
We propose in this section a practical two stage algorithm for computing an estimate $\hat\theta_n$ within a particular accuracy-cost tradeoff. Assuming we have $n$ unlabeled examples, the algorithm begins the first stage by labeling $r$ samples. It then 
estimates $\theta'$ by maximizing the likelihood over the $r$ labeled and $n-r$
unlabeled samples. The estimate $\hat\theta'$ is then used to obtain a plug-in
estimate for the asymptotic accuracy
\correct{{$\widehat{\trace(\Sigma)}$}}{{$\trace(\Sigma)$}}. In the second stage the
algorithm uses the estimate \correct{{$\widehat{\trace{\Sigma}}$}}{{$\widehat{\trace(\Sigma)}$}} to resolve the
tradeoff via  \eqref{eq:tradeoff1}-\eqref{eq:tradeoff3} and determine how many
more labels should be collected. Note that the labels obtained at the first
stage may be used in the second stage as well with no adverse effect. 

The two-stage algorithm spends some initial labeling cost in order to obtain an estimate for the quantitative tradeoff parameters. The final labeling cost, however, is determined in a principled way based on the relative importance of accuracy and labeling cost via \eqref{eq:tradeoff1}-\eqref{eq:tradeoff3}. The selection of the initial number of labels $r$ is important and should be chosen carefully. In particular it should not exceed the total desirable labeling cost. 

We provide some experimental results on the performance of this algorithm in
Figure~\ref{fig:lorem3} (right). It displays box-plots for the differences
between $\trace(\Sigma)$ and $\widehat{\trace(\Sigma)}$ as a function of the
initial labeling cost $r$ for naive Bayes SSL classifier and 20-newsgroups
data. The figure illustrates that the two stage algorithm provides a very
accurate estimation of \correct{{$\trace{\Sigma}$}}{{$\trace(\Sigma)$}} for $r\geq 1000$ which becomes almost
perfect for $r\geq 1300$.   

{\bibliographystyle{plain}
\bibliography{../../common/externalPapers,../../common/groupPapers}}
\end{document}